%% file: main.tex
\documentclass{article}

\usepackage{PRIMEarxiv}
\usepackage{natbib}
\setcitestyle{numbers}
\setcitestyle{square}
\usepackage[utf8]{inputenc} 
\usepackage[T1]{fontenc}    
\usepackage[hidelinks]{hyperref}       
\usepackage{url}            
\usepackage{booktabs}       
\usepackage{amsfonts}       
\usepackage{nicefrac}       
\usepackage{microtype}      
\usepackage{lipsum}
\usepackage{graphicx}
\graphicspath{{figures/}}     
\usepackage{subfig}

\usepackage{amssymb}
\input{commands.tex}

\usepackage{natbib}

\title{On Bridging the Gap between Mean Field and Finite Width in Deep Random Neural Networks with Batch Normalization
}

\author{
  Amir Joudaki \\
  BMI, ETH Zürich \\
  \texttt{amir.joudaki@ethz.ch} \\
   \And
  Hadi Daneshmand\\
  MIT LIDS \\
  \texttt{hdanesh@mit.edu} \\
   \And
  Francis Bach\\
   INRIA-ENS-PSL Paris \\ 
   \texttt{francis.bach@inria.fr}\\
}

\begin{document}
\maketitle

\begin{abstract}
Mean field theory is widely used in the theoretical studies of neural networks. In this paper, we analyze the role of depth in the concentration of mean-field predictions, specifically for deep multilayer perceptron (MLP) with batch normalization (BN) at initialization. By scaling the network width to infinity, it is postulated that the mean-field predictions suffer from layer-wise errors that amplify with depth. 
We demonstrate that BN stabilizes the distribution of representations that avoids the error propagation of mean-field predictions. 
This stabilization, which is characterized by a geometric mixing property, allows us to establish concentration bounds for mean field predictions in infinitely-deep neural networks with a finite width.
\end{abstract}


\input{sections/intro.tex}

\input{sections/motivation.tex}

\input{sections/main_results.tex}

\input{sections/consequences.tex}

\input{appendix/proof_of_theorem.tex}

\input{sections/conclusion.tex}
\bibliographystyle{apalike}  

\bibliography{bibliography}

\appendix 
\input{appendix/lowerbound_GP.tex}

\end{document}

%% file: commands.tex
\usepackage{bm,amsthm}
\usepackage{commath,dsfont,cleveref,nicefrac,tikz,xfrac}
\usepackage{booktabs}
\usepackage{mathtools, amssymb, bbm}

\newtheorem{Theorem}{Theorem}

\newtheorem{remark}{Remark}
\newtheorem{assumption}{Assumption}

\newtheorem {lemma} [Theorem]    {Lemma}
\newtheorem {corollary}  [Theorem]    {Corollary}

\newtheorem {theorem}[Theorem]    {Theorem}

\newcommand{\eps}{{\varepsilon}}
\newcommand{\C}{{G_*}}
\newcommand{\Csa}{G}
\renewcommand{\O}{{O}}

\newcommand{\TV}{{{tv}}}

\newcommand{\nhalf}{{\nicefrac{-1}{2}}}
\newcommand{\half}{{\nicefrac{1}{2}}}
\newcommand{\E}{\mathbb{E}}
\renewcommand{\P}{\mathbb{P}}
\newcommand{\nfrac}{\nicefrac} 
\newcommand{\Expec}[1]{\operatorname{\E}\left[{#1}\right]}
\newcommand{\Prob}[1]{\operatorname{\P}\big\{{#1}\big\}}
\newcommand{\N}{\mathbb{N}}
\newcommand{\R}{\mathbb{R}}

\def\1{\bm{1}}
\def\0{\bm{0}}
\let\norm\relax
\DeclarePairedDelimiter\norm{\|}{\|}

\newcommand{\row}{{\mathrm{row}}}

\newcommand{\F}{{\sigma}}

\newcommand{\distort}{{\gamma}}
\newcommand{\relu}{{\mathrm{relu}}}
\newcommand{\BN}{{\phi}}

\newcommand{\Normal}{\mathcal{N}}
\newcommand{\diag}{\mathrm{diag}}

\newcommand{\tv}[2]{\norm{\mu_{#1}-\mu_{#2}}_{tv}}

\def\eqref#1{equation~\ref{#1}}

\DeclareMathAlphabet{\mathsfit}{\encodingdefault}{\sfdefault}{m}{sl}
\SetMathAlphabet{\mathsfit}{bold}{\encodingdefault}{\sfdefault}{bx}{n}



%% file: sections/intro.tex
\section{Introduction}
There is a growing demand for a theoretical framework to characterize and enhance the robustness, safety, computational and statistical efficiency of neural networks.  Mean field theory, from statistical mechanics,  has provided insights into wide neural networks with an infinite width. Going beyond the microscopic analysis of individual neurons, mean field analysis has revealed collective behaviors of neurons at initialization~\cite{pennington2018emergence,yang2019mean,pennington2017nonlinear}, during training \cite{jacot2018neural,bach2021gradient,lee2019wide}, and post-training \cite{chizat2020implicit,ba2019generalization}.

 This paper specifically focuses on the applications of mean field theory at initialization when the network weights are random. Since the seminal works by~\citet{pmlrv9glorot10a} and~\citet{saxe2013exact} demonstrated the effects of initialization on training, applications of mean field theory at initialization have led to numerous insights. Mean field theory has been leveraged to reveal links between wide neural networks and Gaussian processes~\cite{neal2012bayesian,matthews2018gaussian,jacot2018neural}, study concentration of singular values of input-output Jacobians~\cite{pennington2018emergence,feng2022rank}, and to design activation functions~\cite{klambauer2017self,ramachandran2017searching,li2022neural}. Remarkably,~\citet{xiao2018dynamical} introduce an initialization enabling the training of convolutional networks with $10000$ layers.

Despite the successes of mean field theory, there is an inherent approximation error between the infinite-width mean field regime and the finite width used in practice. \citet{matthews2018gaussian} observe that increasing depth amplifies this approximation error, and when the network is sufficiently deep, the mean-field predictions break down. To control this error propagation,~\citet{matthews2018gaussian} propose to increase the network width proportional to depth. Other studies have similarly suggested scaling the network width and depth to infinity while keeping their ratio depth$/$width constant~\cite{hanin2019finite,li2021future}. In a similar spirit,~\citet{hanin2022correlation} establishes an $O(\text{depth$/$width})$-concentration bound for mean-field predictions.

In the present work, we are fundamentally interested in the following question: Can we achieve bounded mean field error with infinite depth, even when the width is finite? 
We observe that mean-field predictions are very accurate for MLPs with batch normalization. In particular, we observe the errors stabilize with depth. Batch normalization biases the hidden representations towards the mean field solutions, with the deviations bounded by
\begin{align}
 e^{-\alpha\cdot \text{depth}/2} + \text{width}^{\nhalf},
\end{align}
up to constants 
for some $\alpha >0$ that arises from a technical assumption about the underlying dynamics (see the formal statement in Theorem~\ref{thm:main_ESD}).
Thus, the mean-field predictions become quickly accurate after a few layers. While mean-field predictions suffer from  $O(\text{depth}/\text{width})$ concentration bounds without batch normalization \cite{li2022neural}, the established concentration bound only depends on the width for deep neural networks.

Bridging the gap between infinite and finite width analyses, one can use mean-field predictions to explain representations in neural networks with batch normalization.
 ~\citet{yang2019mean} elegantly show a wide range of activations have well-conditioned representations in MLPs with batch normalization for neural networks with infinite widths. The established concentration bound translates this result to neural networks with finite widths.
 
\subsection{Our contributions}
Under the assumptions that the weights are Gaussian and the chain of hidden representations is geometric ergodic, we prove:
\begin{itemize}
    \item The spectrum of hidden representation of BN-MLP stabilises around the mean field predictions for deep neural networks with finite widths. 
    \item The hidden representations of MLPs with BN for wide range of activations, used in practice, are well-conditioned with a non-asymptotic bound for their deviations.
    \item Sufficient conditions to bridge analyses for neural networks with finite and infinite width. 
\end{itemize}


%% file: sections/motivation.tex
\section{Motivation}
Numerous studies~\cite{saxe2013exact,feng2022rank,yang2019mean} have provided valuable insights into the training of deep neural networks by analyzing the input-output Jacobians of neural networks at initialization. For example, ~\citet{feng2022rank} have shown that the rank of the input-output Jacobian of neural networks without normalization at initialization diminishes exponentially with depth. 
The spectrum of Jacobians is closely related to the spectra of Gram matrices. 
A Gram matrix, also referred to as G-matrix, contains the inner products of samples within a batch (\eqref{eq:gram_matrix}). Thus, a degenerate G-matrix for the penultimate layer implies that the outputs are blind to the inputs~\cite{feng2022rank,li2022neural}.  
Diminishing rank in the last hidden layer occurs in various neural architectures, including MLPs~\cite{saxe2013exact}, convolutional networks~\cite{daneshmand2020batch}, and transformers~\cite{dong2021attention}, and leads to ill-conditioning of the input-output Jacobian, which slows training~\cite{daneshmand2021batch,pennington2018emergence,yang2019mean}. ~\citet{saxe2013exact} have shown that avoiding rank collapse can accelerate the training of deep linear networks, making it a focus of theoretical and experimental research~\cite{pennington2018emergence,daneshmand2020batch,daneshmand2021batch}.

Mean field theory has been effectively used to analyze spectra singularities in random deep neural networks \cite{pennington2017nonlinear,pennington2018emergence,xiao2018dynamical,li2022neural,yang2019mean}. In particular, mean field theory provides guidelines to avoid the rank collapse of representation that enhances the training of deep neural networks~\cite{xiao2018dynamical}. Furthermore, mean field theory provides insights into the interplay between the spectra singularities and neural architectures. \cite{yang2019mean} proves batch normalization layers, which is a key component of deep neural networks \cite{ioffe2015batch}, avoid the rank collapse of gram matrices in a mean-field regime. Interestingly, \citet{daneshmand2021batch,daneshmand2020batch} proves that this mean-field analysis accurately holds for neural networks with a finite width and linear activations. For these neural networks, batch normalization iteratively biases the gram matrices to the identity matrix as network depth grows~\cite{daneshmand2021batch} that ensures the non-singularity of gram matrices at deep layers. To the best of our knowledge, this is the only result that proves non-vacuous concentration bounds for mean-field predictions in the standard settings when deep neural networks have finite widths. However, the result of \citet{daneshmand2021batch} is limited to neural networks with linear activations. We characterize sufficient conditions to extend this result to a wide range of non-linear activations.

\section{Problem settings and background}

\paragraph{Notation and terminology.}

$I_n$ denotes the identity matrix of size $n\times n.$ 

$\otimes$ refers to Kronecker product. $\mu_X$ refers to the probability measure of the random variable $X$. We use $f\lesssim g, g \gtrsim f$ and $f = \O(g)$ to denote the existence of an absolute constant $c$ such that $f\le c\; g.$ 
 
$\norm{v}$ for vector $v$ denotes the $L^2$ norm.
$\norm{C}$ for matrix $C$ denotes the $L^2$ operator norm $\norm{C}=\sup_{x\in\R^n}\norm{C x}/\norm{x}$, $\norm{C}_F$ denotes Frobenius norm, and $\kappa(C)$ denotes condition number $\kappa(C)=\norm{C}\norm{C^{-1}}.$ Both $h_{r\cdot}$ and $\row_r(h)$ denote row-vector representation of the $r$-th row of $h.$ 

\paragraph{Setup.}
Le $h_\ell\in \R^{d\times n}$ denote the hidden representation at layer $\ell$, where $n$ corresponds to the size of the mini-batch, and $d$ denotes the width of the network that is kept constant across all layers. The sequence  $\{h_\ell\}$ is a Markov chain as
\begin{align} \label{eq:chain}
h_{\ell+1}:=W_\ell \F\circ\BN(h_\ell), && W_\ell\sim\Normal(0,\nfrac1d)^{d\times d},
\end{align}
where $h_0\in\R^{d\times n}$ is the input batch, $\F$ is the element-wise activation function, and $\BN$ is the batch normalization \cite{ioffe2015batch}, which ensures each row has zero mean and unit variance:
\begin{align*}
\BN(x) = \frac{x - \overline{x}}{\sqrt{\text{Var}(x)}}, && \forall r: \row_r(\BN(h)) = \BN(\row_r(h)).
\end{align*}
The G-matrix $G_\ell$ is defined as the matrix of inner products of hidden representations at layer $\ell$
\begin{align}\label{eq:gram_matrix}
G_\ell := \frac1d (\F\circ\BN(h_\ell))^\top (\F\circ\BN(h_\ell)).
\end{align}

The mean field approximation for $G_\ell$ is defined through the following recurrence \cite{yang2019mean}
\begin{align}\label{eq:MF_recurrence}
    \overline{G}_{\ell+1} = \Expec{(\F \circ \BN(w_\ell))^{\otimes 2}}, \quad w_\ell\sim \Normal(0,\overline{G}_\ell)
\end{align}
where $\overline{G}_{0}=G_0$ for the input G-matrix. Note that mean field Gram matrices $\overline{G}_\ell$ are deterministic, while Gram matrices $G_\ell$ are random for $\ell>0$. 
By casting a stochastic process to a deterministic process, the mean-field approach simplifies the analysis of Gram matrices. Intuitively, one could expect that the recursion step in~\eqref{eq:MF_recurrence} multiple times, it may converge to a fixed point of this equation. Inspired by this intuition, \citet{yang2019mean} study the fixed-points of the recurrence denoted by $\C$, which obeys

\begin{align}\label{eq:MF_stable_C}
    \C = \Expec{(\F\circ\BN(w))^{\otimes 2}}, \quad  w\sim\Normal(0,\C)
\end{align} 
\citet{yang2019mean} characterize stable and unstable fixed points for neural networks with batch normalization. The recurrence in equation~\eqref{eq:MF_recurrence} is attracting in a local neighborhood of a stable fixed point. More interestingly, \citet{yang2019mean} establish even the global stability for networks with linear activations. This global stability does not ensure the convergence of $G_\ell$ (even) to a local neighborhood of $\C$ since the mean-field approximation suffers from an $O(d^{-1/2})$-approximation error in each layer. Increasing depth may amplify this error, in that $G_\ell$ diverges from $\overline{G}_\ell$ for a large $\ell$. We investigate the error propagation with depth for this mean-field prediction.

\section{How can errors propagate through depth?}
\subsection{Empirical observation}
To demonstrate the discrepancy between mean-field analysis and the practical settings of neural networks with finite width, we will use a toy example of a multi-layer perceptron (MLP) without batch normalization and with linear activations. This corresponds to a chain with $\F\circ\BN=id$ in~\eqref{eq:chain}. We will use this toy example to show that the mean-field approximation error amplifies with depth when networks do not have normalization layers and then contrasts this for neural networks with batch normalization. Notably, this is a warm-up illustration and our theoretical analysis is not limited to this example. 

Note that the identity matrix of size $n$ is a mean field fixed point satisfying~\eqref{eq:MF_stable_C}, as it holds trivially $\Expec{w^{\otimes 2}}$ where w$\sim\Normal(0,I_n)$. However, as evidenced in Figure~\ref{fig:rapidly_mixing}, we observe that the Frobenius distance between $G_\ell$ from $\C$ increases at an exponential rate in the number of layers for networks without batch normalization. This means that the mean-field approximation error amplifies with depth, making it necessary to add further refinements for neural networks without normalization layers, as previously reported by ~\citet{li2022neural}.

In contrast, when we repeat the experiment by adding batch normalization to the MLP, we observe that as the depth grows, the G-matrices $G_\ell$ converge to a neighborhood of the mean-field G-matrix $\C$. Instead of error amplification with depth, in BN-MLP the errors stabilize up to a constant. Figure~\ref{fig:rapidly_mixing_width} shows that this constant is inversely proportional to the network width $\sqrt{d}$.
These observations suggest that mean field predictions by~\cite {yang2019mean} are highly accurate for finite width in the presence of batch normalization. On contrary, the existing concentration bounds for neural networks without batch-normalization depend on $O(\text{depth}/\text{width})$, hence breaking in an infinite depth regime. 

\begin{figure}
\subfloat[With/Without BN\label{fig:rapidly_mixing}]{
\includegraphics[width=0.5\textwidth]{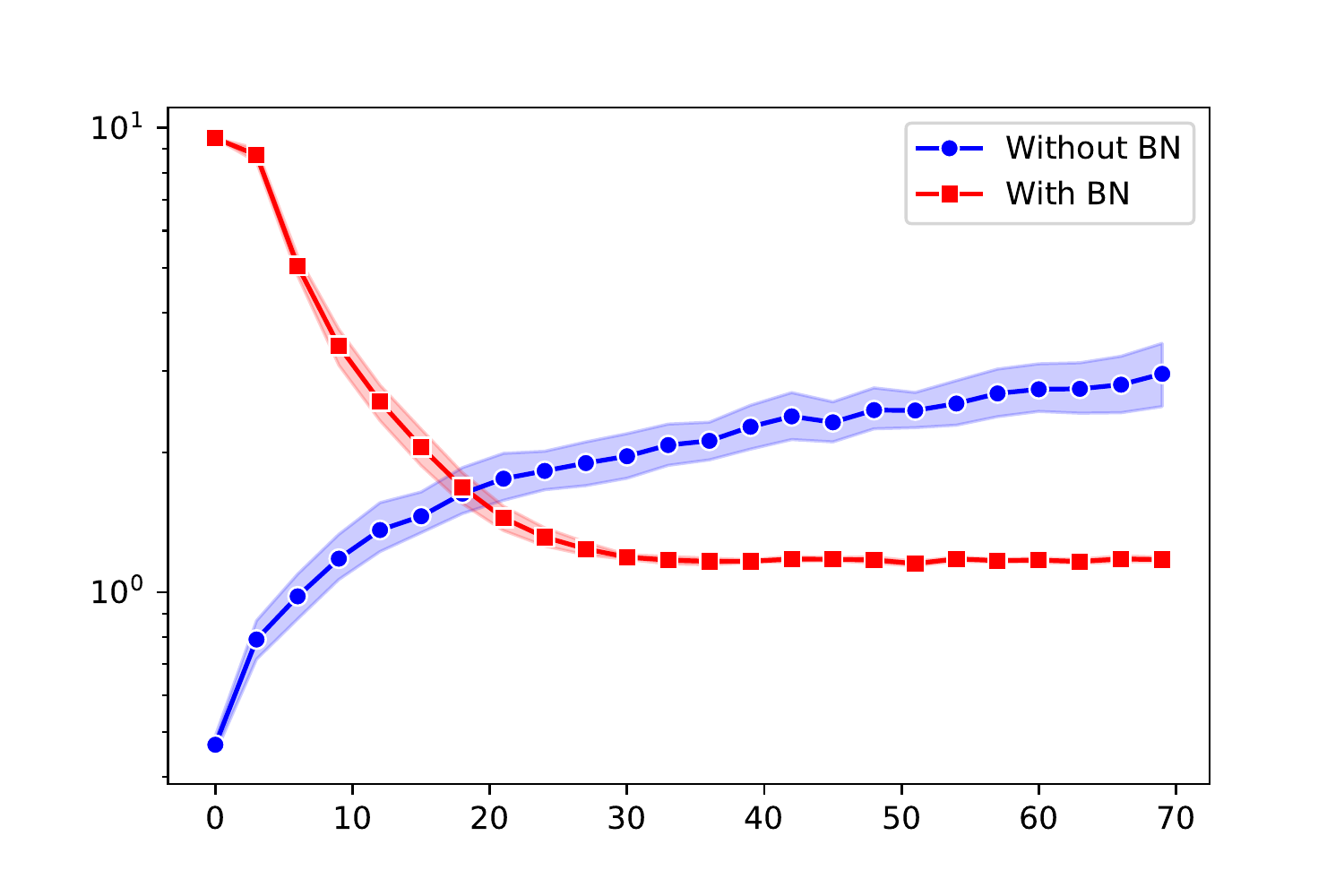}} 
\subfloat[Width effect\label{fig:rapidly_mixing_width}]{
\includegraphics[width=0.5\textwidth]{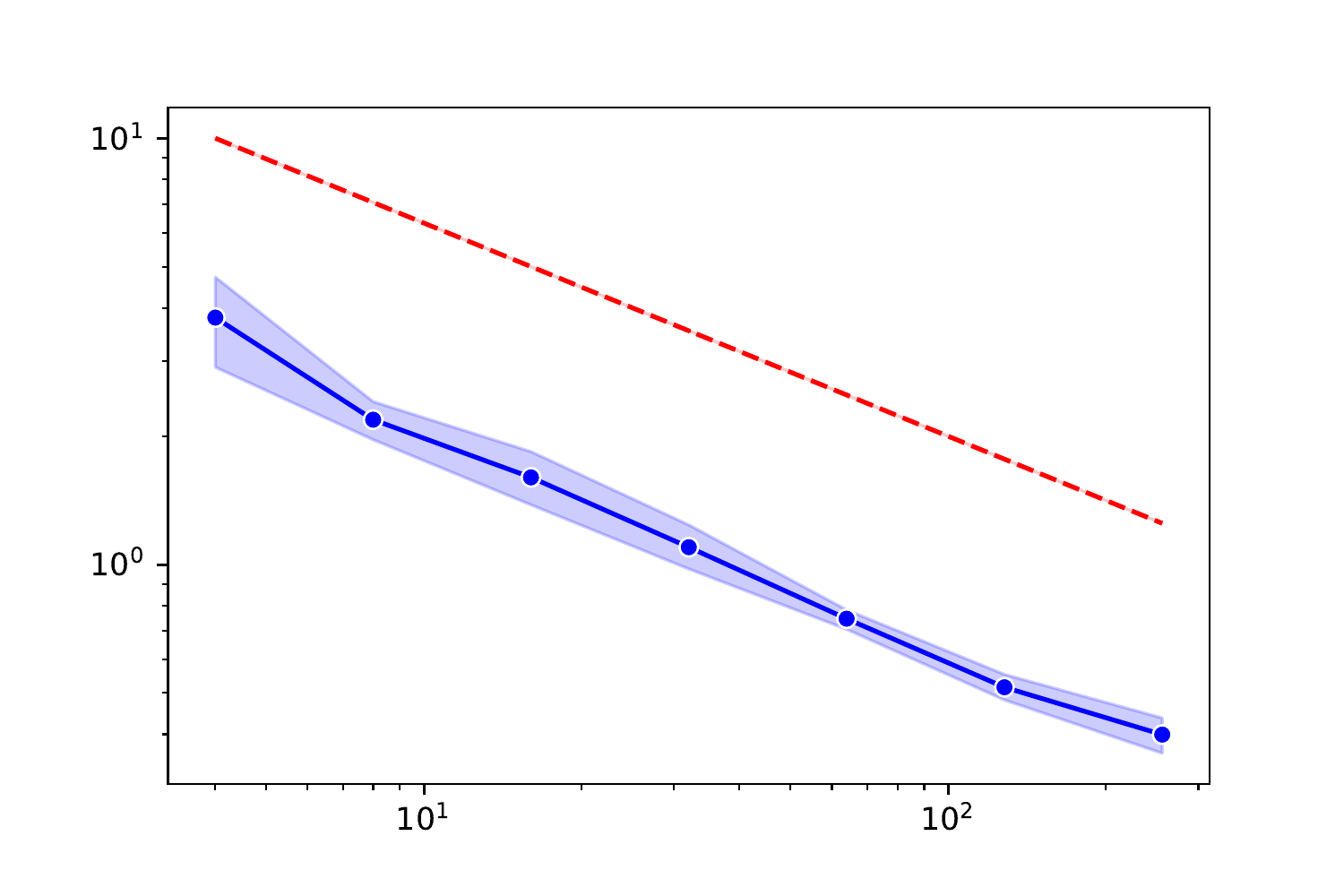}}
 \caption{(a) \footnotesize{\textit{mean field error amplification with(out) batch-normalization.} Horizontal axis: number of layers $\ell$ (linear); vertical axis (log-scale): $\|G_\ell- \C\|_F$, $n=5,d=1000$.  The lines show mean and shades for $90\%$ confidence intervals over $10$ independent simulations.} (b) \footnotesize{Mean field deviations from Gram matrices for BN-MLP with linear activation $ n=10$. The plot specifically shows the average of $\| G_\ell - \C \|_F$ over $40<\ell<600$  on y-axis (log-scale) vs. width on the x-axis (log-scale). The dashed line marks $2n/\sqrt{d}$.}}
\end{figure}
 
    

    

\subsection{Analytical illustration}
In this section, we will analytically investigate the empirical observations on the error propagation with depth.

A mean-field analysis lets $d$ tend to infinity, which implies that G-matrix $W_k^\top W_k$ converges to its expected value, $I_n$, almost surely. Replacing this in mean-field recurrence of \eqref{eq:MF_recurrence} yields the mean-field Gram matrices are the identity matrix across all the layers, namely $\overline{G}_\ell = I_n$ holds for all $\ell$.

For any finite $d$, the mean-field predictions suffers from $O(d^{-1/2})$ approximation in each layer. It is straightforward to bound of deviations of the original gram matrix from the mean-field prediction  using union-type reasoning to prove an overall bound for deviation from $\C$.  However, this approach leads to an accumulation of rates of failure when depth is larger than width $\E \norm{G_\ell - I_n}\le \exp(\nfrac{3n\ell}{2d})$. This formula hints at an interesting interplay between depth and width that acts as an indicator of the discrepancy between mean field predictions and Gram matrices (informally): \begin{align}\label{eq:mean_field_discrepancy_informal}
\text{error propagation} \approx \exp(\frac{ \text{depth}}{\text{width}}).
\end{align}

 This bound, which was derived for the conceptually simple case of a linear MLP, is consistent with the bounds previously reported in the literature~\cite{matthews2018gaussian,hanin2019finite}. In an insightful observation,~\citeauthor{li2022neural} showed that the error of an infinite-width-and-depth network behaves like a geometric Brownian motion. By approximating the dynamics of Gram matrices with a stochastic differential equation (SDE), they were able to accurately predict the degenerate outputs for a vanilla MLP in realistic settings (see section 2.2 and equation 2.6). In fact, ~\citet{li2022neural} leverage their SDE approximation to carefully shape activation functions to achieve non-degenerate Gram matrices in infinite-depth. However, this does not address the stabilizing effects of normalization in the absence of activation shaping, which is the main focus of the current work.

\citet{daneshmand2021batch} mathematically prove the observation for MLP with BN and linear activations. According to this paper, batch normalization iteratively tends the Gram matrices $G_\ell$ to the identity matrix as $\ell$ grows. The identity matrix is the mean-field Gram matrix when tending the network width to infinity~\cite{yang2019mean}. Indeed, \citet{daneshmand2021batch} establish the first concentration bound for mean-field predictions in the presence of batch normalization. However, the result of \citet{daneshmand2021batch} is limited to linear activations. Here, we extend this result to non-linear activations under an assumption from Markov chain theory.

%% file: sections/main_results.tex
\section{Main results}

\subsection{Geometric ergodic assumption} 
While the chain of hidden representations obeys a non-linear stochastic recurrence, the distribution associated with the representation obeys a linear fixed-point iteration determined by the Markov kernel $K$ associated with the chain $h_\ell$. The distribution of $h_\ell$, denoted by $\mu_{\ell},$ obeys 
\begin{align} \label{eq:dist_recurrence}
    \mu_{\ell+1} = T(\mu_{\ell}), \quad T(\mu) := \int K(x,y) d\mu(y).
\end{align}
The fixed-points of the above equation are invariant distributions of the chain, which we denote by $\mu_*$.  Recall that total variation for distribution over $d\times n$ matrices can be defined as $\tv{X}{Y}:=\sup_{A\subseteq\R^{d\times n}} |\mu_X(A)-\mu_Y(A)|.$
Remarkable, $\norm{ \mu_\ell - \mu_*} \leq \| \mu_{\ell-1} - \mu_* \|_{tv}$ holds for all $\ell$.
We assume the chain obeys a strong property that ensures the convergence to a unique invariant distribution.

\begin{assumption}[Geometric ergodicity]\label{ass:rapid_mixing}
 We assume the chain of hidden representations admits a unique invariant distribution. Furthermore, there is constant $\alpha$ ($\alpha>0$) such that
\begin{align*}
\tv{\ell}{*} \le (1-\alpha)^\ell \tv{0}{*},
\end{align*}
holds almost surely for all $h_0$.
\end{assumption}

The geometric ergodic property is established for various Markov chains, such as the Gibbs sampler, state-space models~\cite{eberle2009markov}, hierarchical Poisson models~\cite{rosenthal1995minorization}, and Markov chain Monte Carlo samplers~\cite{jones2001honest}. We conjecture that the chain of hidden representations is geometric ergodic. In particular, the sufficient condition of  \citet{doeblin1938deux} potentially holds for the chain of hidden representations. Remarkably,  Doeblin's conditions hold when Markov chains can explore the entire state space~\cite{eberle2009markov}.
Intuitively speaking, when $h_\ell$ has full rank, the Gaussian product $W_\ell h_\ell$ may explore the entire $\R^{d\times n}$. 
We leave the proof of the geometric ergodic property to related research in Markov chain theory.

\subsection{Main results for standard activations}
Under geometric ergodicity, the next theorem proves a spectral concentration for $G_\ell$ around $\C$, for activations commonly used in practice denoted by the set $\mathcal{F}:=\{\relu,\tanh,\mathrm{sigmoid},\sin,\mathrm{selu},\mathrm{celu}\}$.  

\begin{theorem}[Spectral concentration of BN-MLP]\label{thm:main_ESD}
Consider the Markov chain $\{h_\ell\}$ for the BN-MLP chain with activation $\F\in \mathcal{F}$, and Gram matrices $\{G_\ell\}.$ Let $\C$ denote the stable Gram matrix, and $\lambda_i$'s and $\lambda^*_i$'s eigenvalues of $G_\ell$ and $\C$ respectively (descending). Assuming that $\{h_\ell\}$ obeys Assumption~\ref{ass:rapid_mixing} with $\alpha\in(0,1],$ and $\C$ is non-degenerate, for sufficiently large $d\gtrsim n^2\norm{\C^{-1}},$ define
If input G-matrix $G_0$ is non-degenerate, we have
\begin{align}
\Prob{\bigwedge_{i=1}^n\big\lvert{\lambda_i}/{\lambda_i^*}-1\big\lvert\gtrsim t(e^{-\frac{\alpha\ell}{2}}+\eps\ln\frac{1}{\eps}) } \ge 1-t^2, && \eps:=\frac{n\norm{\C^{-1/2}}}{\sqrt{d}}.
\end{align}
\end{theorem}

Let us review the implications of the Theorem~\ref{thm:main_ESD}. Note that this theorem captures the multiplicative deviations of $\lambda_i$'s from stable $\lambda_i^*$'s: with probability $1-t^{-2}$ we have a bound on $|\lambda_i/\lambda_i^*-1|,$ simultaneously for all $i.$  We can decompose deviations $\eps$ as (informally): 
\begin{align*}
\text{error} \approx \underbrace{\exp(-\nfrac{\alpha}{2}\cdot\text{depth})}_{\text{transitory}} + \alpha^{\nhalf}\underbrace{\frac{\text{batch-size}}{\text{width}^{1/2}}}_{\text{stationary}}.
\end{align*}
The first term $e^{-\alpha\text{depth}}$ implies the convergence to a stable state at rate $(1-\alpha)$, and the second part as the stationary term that becomes dominant after a logarithmic number of layers $\Omega( \log(\text{width}/\text{batch-size}))$. Thus, the established concentration only depends on the network width when the number layers are $\Omega(\log(\text{width}/\text{batch-size}))$. This is in stark contrast with the concentration bounds for neural networks without batch normalization that become vacuous with depth~\cite{hanin2019finite,hanin2022correlation}. While  activation shaping requires solving an SDE to track the dynamics of Gram matrix in neural networks without BN, BN-MLP only relies on the mean-field prediction $\C $ computed in closed-form by~\citet{yang2019mean} without shaping activations.

The reader may contrast the above result with~\eqref{eq:mean_field_discrepancy_informal}, where the mean field predictions suffer from an error growing at exponential rate with depth. The main difference between the two chains is the presence of normalization layers that stabilize the spectrum of Gram matrices.

 Depending on the activation function, one might not be able to analytically derive the constants in the expression for $\C$. Yet, we can show the concentration regardless of having access to $\C$: If we pass two inputs through the same MLP, then the spectra of their Gram matrices contract to a local neighbourhood of each-other.

\begin{corollary}\label{cor:ESD_coupling}
In the same setting as Theorem~\ref{thm:main_ESD}, let $\{h_\ell\}$ and $\{h'_\ell\}$ denoting representations of a random network matrices for two different inputs, with Gram matrices $\{G_\ell\}$ and $\{G'_\ell\}$ respectively, where both $G_0$ and $G'_0$ are non-degenerate, then
\begin{align}
\Prob{\bigwedge_{i=1}^n\big\lvert{\lambda_i}/{\lambda_i'}-1\big\lvert\gtrsim t(e^{-\frac{\alpha\ell}{2}}+\eps\ln\frac{1}{\eps}) } \ge 1-t^2, && \eps:=\frac{n\norm{\C^{-1/2}}}{\sqrt{d}}.
\end{align}
where $\lambda_i$'s and $\lambda'_i$'s are eigenvalues of $G_\ell$ and $G'_\ell$ respectively (descending).
\end{corollary}

\begin{remark}
 We shall note the asymmetry of definition of $x=(\lambda_i/\lambda_i'-1)^2$ from $y=( \lambda_i'/\lambda_i-1)^2$, despite the fact that the definition of two chains is symmetric. However, $x$ and $y$ can only be apart by a constant $\nfrac{x}{2}\le y \le 2x,$ and thus are ``quasi-symmetric.''
 \end{remark}
\begin{remark}
One advantage of Corollary~\ref{cor:ESD_coupling} is that it is highly predictive without any prior knowledge of $\C.$ In other words, $\sum_i |\lambda_i/\lambda_i'-1|$ can serve as a numerical certificate for the stability of the spectrum. Similarly, corollary~\ref{cor:coupling} can be seen as a numerical certificate for moment stability in Theorem~\ref{thm:concentration}, when one does not have access to the structure of $\C.$
\end{remark}

\subsection{General activations}

Theorem~\ref{thm:main_ESD} extends to a broader family of activations with an adapted concentration bound.

\begin{theorem}[BN-MLP concentration]\label{thm:concentration}
Consider a Markov chain $\{h_\ell\}$ with Gram matrices $\{G_\ell\},$ and $\C$ is a stable mean field G-matrix. Define $\distort$ as the distortion of activation $\F$ induced on the sphere with $\sqrt{n}$-radius: 
\begin{align}\label{eq:distortion_def}
\distort:=\sup_{x\in\mathbb{S}} \frac{\norm{\F(x)}}{\norm{x}}, && \mathbb{S}:=\{x\in\R^n: \norm{x}=\sqrt{n}\}.
\end{align}
If chain $\{h_\ell\}$ obey Assumption~\ref{ass:rapid_mixing} with  $\alpha \in (0,1]$,  if input G-matrix $G_0$ and stable G-matrix $\C$ are non-degenerate, we have
\begin{align}\label{eq:concentration_bound}
\Prob{ \norm{\C^{-1}G_\ell - I_n}_F\ge t(e^{-\frac{\alpha\ell}{2}}+\eps\ln\nfrac{1}{\eps}) } \lesssim t^{-2}, && \eps:=\frac{n\gamma\norm{\C^{\nhalf}}}{\sqrt{d}}.
\end{align}
\end{theorem}

\begin{remark}
Note that for all $\F\in\mathcal{F},$ there is absolute constant $C, $ such that $|\F(x)|\le C |x|.$ This implies that all common activations $\F\in\mathcal{F}$ have constant distortion $\distort = O(1).$
\end{remark}

Theorem~\ref{thm:concentration} implies the following contractive property. 
\begin{corollary}\label{cor:coupling}
In same setting as Theorem~\ref{thm:concentration}, two gram matrices $\{G_\ell'\}$ and $\{G_\ell\}$ of hidden representations for two different inputs (at the same neural network), if $G_0$ and $G'_0$ are non-degenerate, it holds
\begin{align}
\Prob{\norm{G_\ell^{-1} G'_\ell-I_n}_F \ge t\eps } \lesssim t^{-2},&& \eps:=\frac{n\gamma\norm{\C^{\nhalf}}}{\sqrt{d}}.
\end{align}
where $\eps$ is defined in Theorem~\ref{thm:concentration}.
\end{corollary}


\paragraph{Proof sketch of Theorem~\ref{thm:concentration}.}

We first construct an approximate invariant distribution, associated with $T$ defined in \eqref{eq:dist_recurrence}. To construct such distribution, we leverage the mean-field Gram matrix to construct input $h_0\in\R^{d\times n},$ with rows drawn i.i.d. from
$\row_r(h_0)\sim \Normal(0,\C)$. The next lemma proves the law of $h_0$ denoted by $\mu_0$ does not change much under $T$. 

\begin{lemma}\label{lem:tv_one_step}
With the $\distort$ defined in~\eqref{eq:distortion_def}, we have
\begin{align}
\| T(\mu_0)- \mu_0 \|_{tv}\lesssim \frac{n^2\gamma^2\norm{\C^{-1}}}{d}=:\eps^2.
\end{align}
\end{lemma}
The proof of the last lemma is based on the fixed-point property of $\C$. Using the last lemma together with
Assumption~\ref{ass:rapid_mixing}, we prove that $\mu_0$ is in a $tv$-ball around the invariant distribution $\mu_*$. Under this assumption, we have 
\begin{equation}
\begin{aligned}
    \| T(\mu_0)-T(\mu_*)\|_{tv} &= \|T(\mu_0)-\mu_*\|_{tv} \\
    & \leq (1-\alpha) \|\mu_0- \mu_*\|_{tv}.
\end{aligned}
\end{equation}
where we used the invariant property of $\mu_*$ in the above equation. 
Using triangular inequality, we get 
\begin{equation}
\begin{aligned}
    \| T(\mu_0) - \mu_0\|_{tv} & = \| T(\mu_0) -\mu_* +\mu_*- \mu_0\|_{tv} \\ 
    & \geq \| \mu_*- \mu_0\|_{tv} -\|T(\mu_0) - \mu_*\|_{tv} \\ 
    & \geq \alpha \| \mu_{0}-\mu_* \|
\end{aligned}
\end{equation}
Plugging the bound from the last lemma into the above inequality concludes $\mu_0$ lies in within a tv-ball around $\mu_*$ with radius $\eps^2/\alpha$. This concludes the proof: Since the chain is geometric ergodic, the distribution $\mu_\ell$ converges to $\mu_*$ at an exponential rate concentrated around $\mu_0$.

%% file: sections/consequences.tex
\section{Consequences}\label{sec:validations}

\subsection{Well-conditioned stable Gram and its implications}
Thus far, we have spoken of the existence, but not the structure of the mean-field Gram matrix $\C$. \citet{yang2019mean} find a mean field approximation by scaling width to infinity and solving~$\C$ solution for ~\eqref{eq:MF_recurrence}.
They introduce two types of fixed points, stable (called BSB1) and unstable (referred as BSB2). $G_*$ is the stable fixed-points.   Leveraging symmetry structure of neural nets, \citet{yang2019mean} show that $\C$ has the following form:
\begin{align}\label{eq:stable_G_symmetric}
   \C = b^*( (1-c^*)I_n + c^* \1_{n\times n}). 
\end{align}

In the special case of common activations, $\F\in\mathcal{F},$ ~\citet{yang2019mean} show that stable Gram matrices are well-conditioned $\norm{\C^{-1}}=O(1),$ and the following spectral concentration hold:
\begin{corollary}[Spectral concentration, well conditioned $\C$]\label{cor:ESD}
In same setting as Theorem~\ref{thm:main_ESD}, for $\eps:=n/\sqrt{d},$ and all $t\ge 0,$ if stable G-matrix is well-conditioned $\norm{\C^{-1}}=O(1),$ and $\eps< 1,$ we have
\begin{align}
    \P\left\{\sqrt{\textstyle\sum_i^n (\lambda_i-\lambda_i^*})^2\gtrsim t(e^{-\alpha\ell/2}+\eps\ln\frac{1}{\eps})\right\}\le t^2.
\end{align}
\end{corollary}

Similarly, if $\C$ is well-conditioned, Theorem~\ref{thm:concentration} implies the following concentration of G-matrices around stable G-matrix $\C$:
\begin{corollary}\label{cor:concentration_diff}
In the same setting as Theorem~\ref{thm:concentration}, for $\eps:=n/\sqrt{d},$ and $t\ge 0,$ if stable G-matrix obeys $\norm{\C^{-1}}=O(1),$ and $\eps<1,$ we have
\begin{align}
    \P\left\{\norm{G_\ell-\C}_F\gtrsim t(e^{-\alpha\ell/2}+\eps\ln\frac{1}{\eps})\right\} \lesssim t^2.
\end{align}
\end{corollary}


Given this particular structure of $\C$ in~\eqref{eq:stable_G_symmetric}, we can conclude $\lambda_i$'s have a very particular structure, with all $\lambda^*_2=\dots=\lambda^*_n = b^*(1-c^*),$ and the largest eigenvalue is $\lambda^*_1=c^* n$.  This implies a more explicit characterization of the spectra of G-matrices in deep neural networks, stated in the next corollary.

\begin{corollary}\label{cor:ESD_spect}
    Suppose that $\C$ in Theorem~\ref{thm:main_ESD} is the stable fixed-point of the mean-field regime and is well-conditioned $\norm{\C^{-1}}=O(1)$; then for sufficiently deep layers ${\ell\gtrsim\log(d/\alpha),}$ $(n-1)$ eigenvalues of $G_\ell$ lie between $c(1-\O(\sqrt{n/d}))$ and $c(1+\O(\sqrt{n/d})),$ with high probability for a constant $c$ that only depends on $ \F\in\mathcal{F}$.
\end{corollary}

\begin{figure}
\subfloat[relu]{\includegraphics[width = 3in]{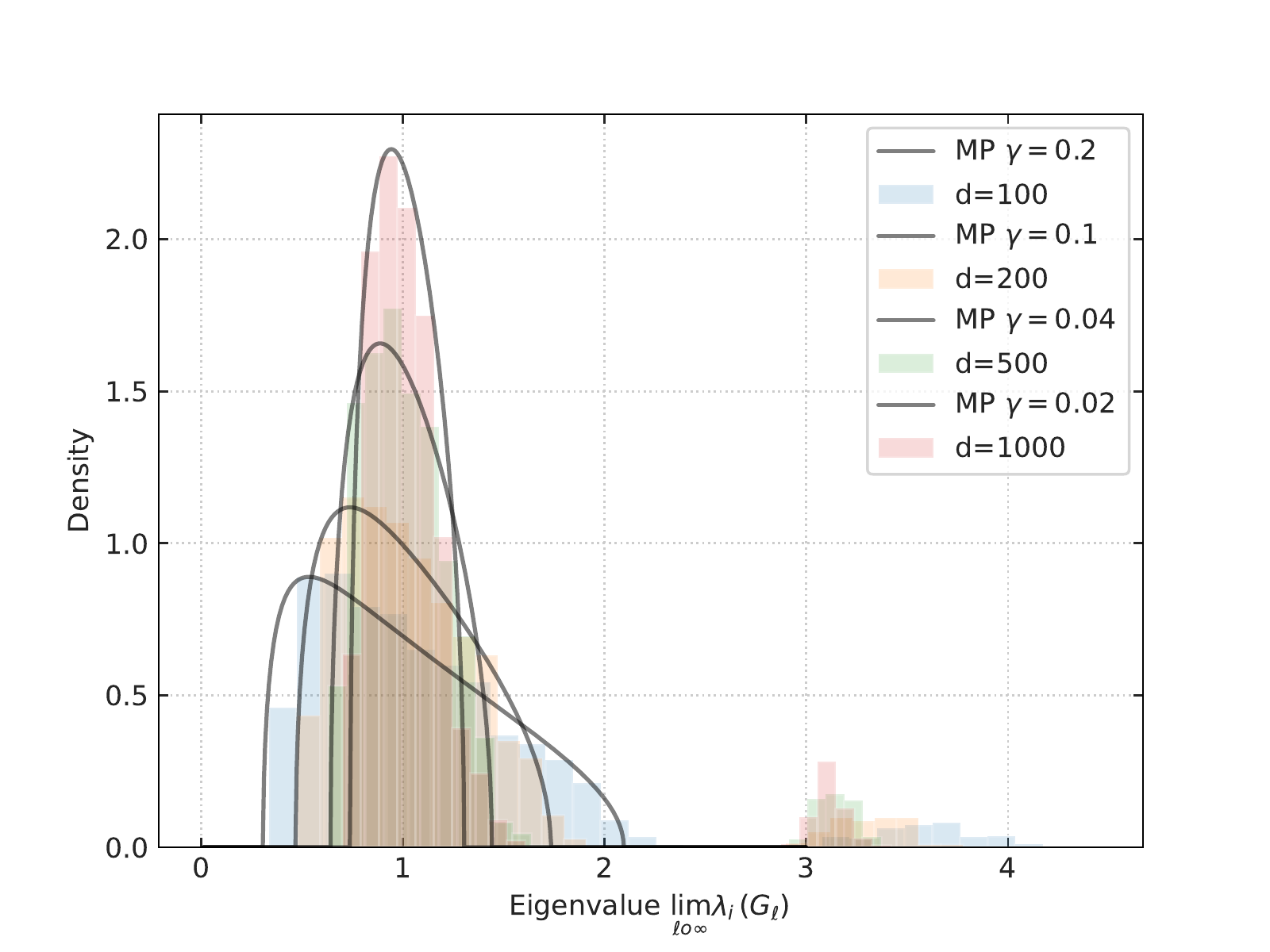}} 
\subfloat[selu]{\includegraphics[width = 3in]{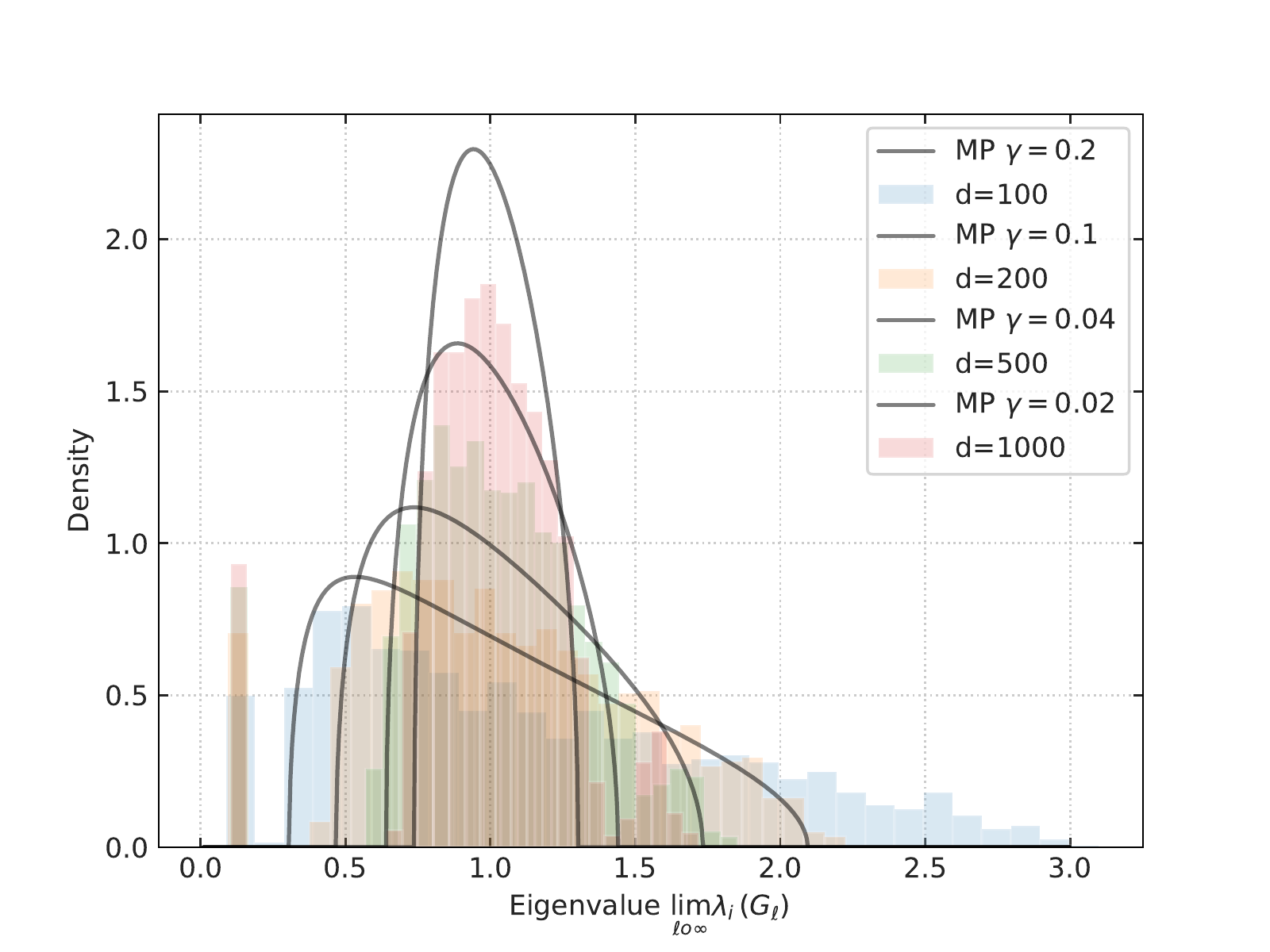}}\\
\subfloat[celu]{\includegraphics[width = 3in]{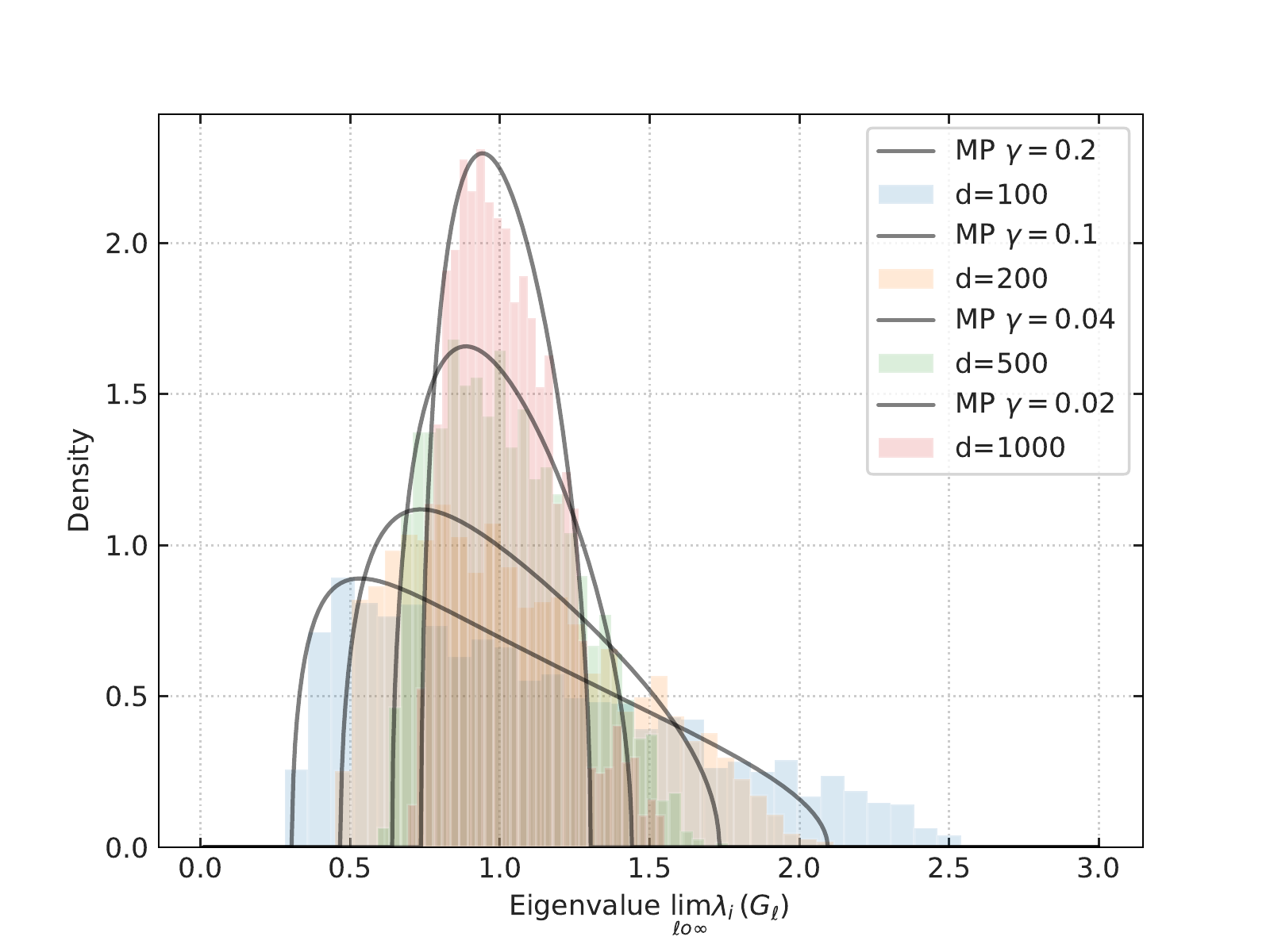}}
\subfloat[tanh]{\includegraphics[width = 3in]{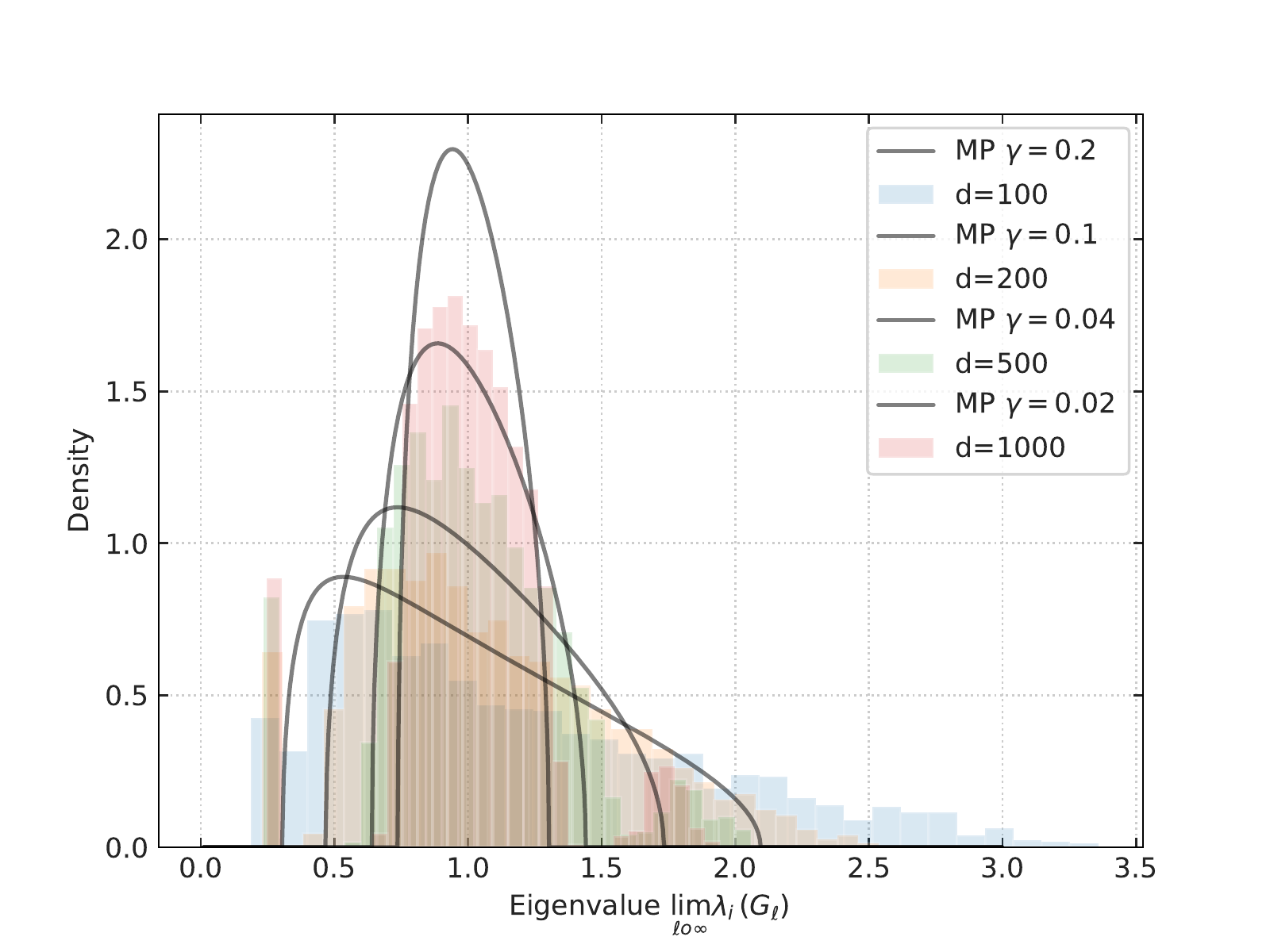}} \\
\subfloat[sin]{\includegraphics[width = 3in]{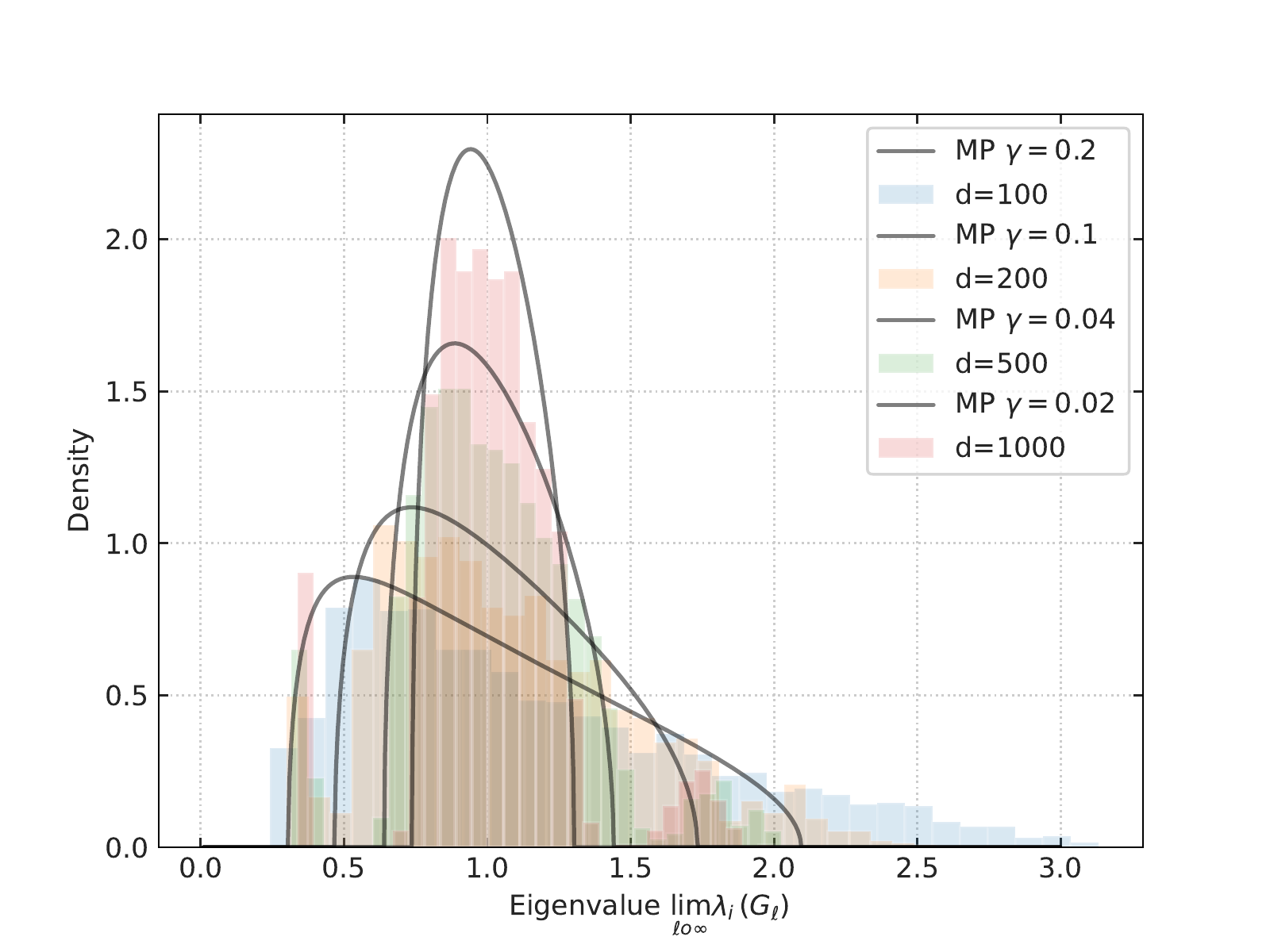}} 
\subfloat[sigmoid]{\includegraphics[width = 3in]{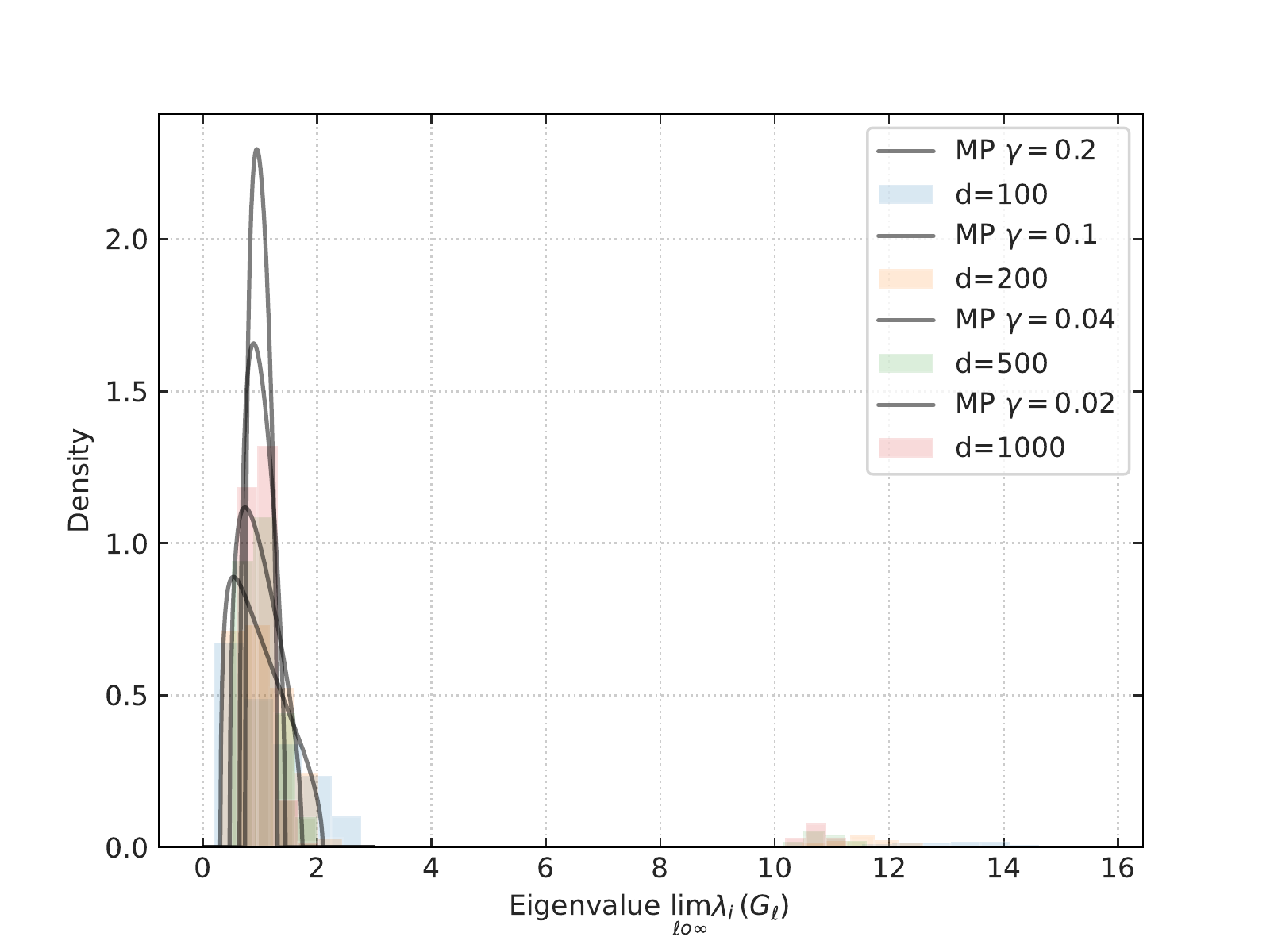}} \\
 \caption{BN-MLP with $n=20,d=500$, $\ell=10$: histogram shows the empirical distribution of $\lambda_i(G_\ell),$ for $\F=\relu$ and $\F=\{\tanh,\sin,\relu,\mathrm{sigmoid},\mathrm{selu}\}$ The black curve marks the Marchenko-Pastur distribution with $\gamma={n/d}.$ The eigenvalues are normalized by their medians in this plot. }
 \label{fig:MP_ESD}
\end{figure}


The previous corollary characterizes the ``bulk'' of eigenvalues of~$G_\ell$: since the Marchenko-Pastur distribution has support at $[1-\sqrt{\gamma},1+\sqrt{\gamma}],$ which implies the eigenalues are supported in the range $[1-\sqrt{n/d},1+\sqrt{n/d}]$ (up to some scaling).  This is a direct consequence of the Marchenko-Pastur law about distribution of eigenvalues of Wishart matrices~\cite{bai1998no,tao2012topics}. In fact, the eigenvalue distribution of $G_\ell$, for sufficiently deep layers, accurately follows Marchenko-Pastur distribution with $\gamma={n/d}.$ Figure~\ref{fig:MP_ESD} shows that bound is remarkably accurate for the spectra. 

We shall also remark that only one eigenvalue is most affected by the activation, while the other $n-1$ behave qualitatively similar. Namely, we can see that $\F=\relu$ induces a single large eigenvalue that is larger that corresponds to the $\1_n$ direction, due to the fact that all outputs of $\relu$ are positive.
We repeat the experiment ($n=20,\ell=10$) in Figure~\ref{fig:MP_ESD} for various activations and various widths $d=100,200,500,1000$, and observe similar results. The results validate tighter concentration bounds with $d$  established in the main Theorem.   


%% file: appendix/proof_of_theorem.tex
\section{Proof of main theorem}\label{app:main_theorem_proof}

\subsection*{Main Tools for Concentration: Total Variation of Gaussians and Matrix Bernstein Inequality}

In this section, we discuss the main tools used for our concentration analysis. In particular, we focus on the use of total variation and the matrix Bernstein inequality.

First, we consider total variation, which is a measure of the difference in probability between the laws of two random variables. For multivariate Gaussian distributions, total variation can be controlled by simple upper and lower bounds, as stated in Theorem 1 of ~\cite{devroye2018total}:
\begin{align}\label{eq:gassian_tv_bound}
\frac{1}{100} \le \frac{\tv{X}{Y}}{\sum_i \lambda_i^2} \le \frac{3}{2}, && X\sim\Normal(C_1),Y\sim\Normal(C_2)
\end{align}
where $\lambda_i$s are eigenvalues of $C_1^{-1}C_2-I_n.$

The second tool that we use is the matrix Bernstein inequality, as stated in Theorem 1.4 of ~\cite{tropp2012user}:
\begin{align}\label{eq:matrix_bernstein}\tag{Matrix Bernstein}
\Prob{\norm{\sum_k X_k}\ge t} \le n e^{-\frac{\nfrac{t}{2}}{\sigma^2 + \nfrac{Rt}{3}}},
\end{align}
where $\{X_k\}_{k\in\N},$ are iid self-adjoint matrices, centered $\E X_k=0$ and have bounded second moments $\sigma^2=\norm{\E X_k X_k^\top}$, with universal operator bound $\lambda_1(X_k)\le R$.

With these tools, we are able to bridge the gap between mean field analysis and finite width neural networks with batch normalization, which will be discussed in the following sections.

\subsection*{Linear Algebra Recap} 
Given a matrix $M\in \R^{n\times m}$, its reduced singular value decomposition (SVD) refers to orthogonal matrices $U\in\R^{m\times r}$, $V\in\R^{n\times r}$ such that $U^\top U = V^\top V = I_r$ and a diagonal matrix $S\in\R^{r\times r}$, such that $M = U S V^\top$, with $U$ and $V$ as eigenvectors, and $\diag(S)=(s_1,\dots,s_n)$ denoting the singular values $s_1\ge \dots \ge s_n\ge 0$. For symmetric matrices $M$, this simplifies to eigenvalue decomposition $M = U \Lambda U^\top$, where $U$ are eigenvectors, respectively, and $\diag(\Lambda)$ contains the singular values $\lambda_1\ge \dots\ge \lambda_n$. If all eigenvalues of $M$ are positive (non-negative), it is positive definite (semi-definite), denoted by $M\succ 0$ ($M \succeq 0$). For positive definite $M$, its powers are defined as $M^\alpha = U S^\alpha U^\top$. We use the notation $\norm{C}$ to denote the $L^2$ operator norm of any operator $T\colon\R^n\to\R^n$: $\norm{T}=\sup_{x\in\R^n}\norm{T(x)}/\norm{x}$. In the special case of $T(x)=M x$ for a positive definite $M\in\R^{n\times n}$, this is equal to its largest singular value $s_1(M)$. We use $\kappa(C)$ to denote condition number $\kappa(C)=\norm{C}\norm{C^{-1}},$ defined for non-singular matrices. 

\subsection{Concentration on the $\sqrt{n}$-sphere}
We provide a lemma that characterizes the deviation between the Gram  and the second moment matrices, and demonstrate how this deviation is controlled by the distortion of the network: 
Recall the definition of distortion $\distort = \sup_{x\in\mathbb{S}} \norm{\F(x)}/\norm{x},$ where $\mathbb{S}=\{x\in\R^n:\norm{x}=\sqrt{n}\}.$

\begin{lemma}\label{lem:concentration}
Let $\mu$ be a probability mass in $\R^n$ within the $\sqrt{n}$-sphere: $\Prob{\norm{z}\le \sqrt{n}}=1$, and for activation $\F:\R\to\R$, $C:=\E_\mu \F(z) \F(z)^\top,$ is positive definite $C\succ 0$, and $\F(z)$ is defined point-wise $\F(z)$ with bounded distortion $\distort$.
Given iid samples $z_1,\dots, z_d\sim \mu$, define Gram matrix $ \Csa:=\frac1d \sum_{k=1}^d z_k z_k^\top$. For any constant $t\ge 0$ it holds:
\begin{align}
\Prob{\norm{C^{-1} \Csa - I_n}_F^2\ge t } \le n e^{- \frac{ \nfrac{t d}{2} }{\distort^2 \norm{C^{-1}} n^2 (1+ \sqrt{t/n})}}.
\end{align}
\end{lemma}

\begin{proof}[Proof of Lemma~\ref{lem:concentration}]
Note that by fact 1, the spectral distribution of $C^{-1} \Csa$ is same as $C^{\nhalf} \Csa C^{\nhalf}$, therefore, we can restate the bound in terms of $\norm{C^\nhalf \Csa C^\nhalf -I_n}_F^2$.
Define $y_k:= C^{\nhalf} z_k$, Define matrices $X_k:=y_k y_k^\top - I_n$. Observe that $y_k$ obeys $\E y_k y_k^\top = I_n$. We can conclude that matrix sequence $\{X_k\}_k$ are centered $\E X_k=0$ and self-adjoint. Therefore, if we can bound matrix second moments $\sigma^2:=\norm{\sum_{k}^d \E X_k^2}$, and uniformly bound all matrices in the sequence $\norm{X_k}\le R$, we can apply matrix Bernstein inequality (Theorem 1.4 of~\cite{tropp2012user}). 

In order to derive a bound for $R$ and $\sigma^2$, let us find a universal for bound $\norm{y_k}$. Recall the definition $y_k = C^\nhalf \F(z_k)$. Observe that $\norm{y_k}\le \norm{C^{-1/2}}\norm{\F(z_k)}$. Because of the assumption that probability mass $\mu$ is almost surely in $\sqrt{n}$-radius $n$-sphere, we can use the definition of distortion $\distort$ to conclude that it holds almost surely $\norm{y_k}\le \distort \norm{C^{-1/2}}\sqrt{n}$. 

1) Universal bound $R$: Recall the definition $X_k = y_k y_k^\top - I$. Note that the eigenvalues of $X_k$ are $\norm{y_k}^2-1$ and $-1$ which is repeated $n-1$ times. Therefore, we can write $\norm{X_k}\le \max(\norm{y_k}^2,1)$, which is almost surely bounded by $\norm{X_k}\le \max(\distort^2 \norm{C^{-1}} n, 1)$. Define $R:=\max(\distort^2 \norm{C^{-1}} n, 1)$. 

2) Bounding $\sigma^2$: Because $X_k$'s are identically distributed, the moment reduces to to to a single matrix in the sequence $\norm{\sum_k^d\E X_k^2} = d\norm{\E X_1^2}$. We can again expand the definition of $X_1$. We have (dropping the $_1$ index for ease of notation):
\begin{align}
    &\norm{\E X^2} = \norm{\Expec{(y y^\top-I_n)^2}} \\
   & = \norm{\Expec{y y^\top y y^\top - 2 y y^\top + I_n}}\\
   & = \norm{\Expec{\norm{y}^2 y y^\top} - 2\Expec{y y^\top} + I_n  }\\
   & \le \norm{R \Expec{y y^\top} - I_n  }\\
   & = \norm{R\cdot I_n - I_n} \\
   & = R - 1
\end{align}
where in the last line we used the fact that $R\ge 1$. Therefore, we have $\sigma^2 = d (R-1)\le d $. 

Plugging values of $\sigma^2$ and $R$ in Matrix Bernstein inequality we have, and define $Z=\nfrac1d \sum_k^d X_k$. We have
\begin{align}
&\Prob{\norm{Z}\ge t} \le n\exp\left(- \frac{ \nfrac{t^2 d}{2} }{(R-1) + \nfrac{t R}{3}}\right) 
\end{align}
Observe that $\norm{Z}_F^2\le n\norm{Z}^2$. Therefore:
\begin{align}
    \Prob{\norm{Z}_F^2\ge n t^2} \le n\exp\left(- \frac{ \nfrac{t^2 d}{2} }{(R-1) + \nfrac{t R}{3}}\right)
\end{align}
We can restate the inequality by change of variables $t':=n t^2$:
\begin{align}
    \Prob{\norm{Z}_F^2\ge t' } \le n\exp\left(- \frac{ \nfrac{t'd}{2} }{(R-1) n + \sqrt{t'/n} R n/3}\right)
\end{align}
Assuming that $\distort$ is a constant, that $n$ is sufficiently large we can replace $R$ by $\norm{C^{-1}}\distort^2 n$, we have
\begin{align}
    \Prob{\norm{Z}_F^2\ge t' }\le n\exp\left(- \frac{ \nfrac{t'd}{2} }{\distort^2 \norm{C^{-1}}n^2 (1+ \nfrac23\sqrt{t'/n})}\right),
\end{align}
where $\nfrac23\sqrt{t'/n} =o(n)$, which converges to $0$ for large $n$. We can again relabel $t'$ by $t$, and plug the definitions of $Z$ and $X_k$, to conclude the claim of the lemma. 

\end{proof}

Remarkably, the Gaussian multi-variate distribution has a lower bound that that can be used to prove that probability of of being close to stable moment is also upper bounded by the total variation. :
\begin{lemma}\label{lem:prob_based_tv}
    In the same setting as Lemma~\ref{lem:tv_one_step}, we have 
    \begin{align}
        \Prob{\norm{\C^{-1}\Csa-I_n}\ge t} &\le \frac{100}{t} \TV(h',h) 
    \end{align}
\end{lemma}

\begin{proof}[Proof of Lemma~\ref{lem:prob_based_tv}]
Define $N_t:=\{M\in R^{d\times n}: \norm{C^{-1}M^\top M-I_n}\le nt\}$. Using th
\begin{align}
\TV(h',h) 
&\ge \Prob{M\notin N_t}\inf_{M\in N_t}\TV(W M, W X) \\
&= \Prob{M \notin N_t} \frac{1}{100}\inf_{M\notin N_t}\norm{C^{-1}M^\top M-I_n}_F^2 \\
&= \Prob{M\notin N_t} \frac{1}{100} nt
\end{align}
which we can restate as
\begin{align}
    \Prob{\norm{C^{-1}M^\top M-I_n}\ge nt} &\le \frac{100}{nt} \TV(h',h) 
\end{align}
\end{proof}

\subsection{Bounding total variation for one step of Markov chain}
While Lemma~\ref{lem:concentration} characterizes the deviation between the Gram matrix and its expectation, this only provides a conditional guarantee. The following lemma goes one step further to bound the total variation of product of previous hidden layer with Gaussian weight matrix.

\begin{lemma}[Restating Lemma~\ref{lem:tv_one_step} from main text]
Let $\C$ denote the stable moment of the joint operator $\F\circ\BN\colon\R^n\to\R^n$, and let $\distort$ denote the distortion of $\F$.
Construct $h\in\R^{d\times n}$ by drawing its rows from $\Normal(0,\C).$ Given $W\sim\Normal(0,1/d)^{d\times d},$ define $h':= W \F\circ\BN ( h).$ It holds
\begin{align}\label{eq:K}
\eps^2=\frac{3n^2\norm{\C^{-1}}\distort^2}{d}\ln(\frac{d}{3n^2\distort^2\norm{\C^{-1}}}),
    &&\TV(h', h)\le \eps^2.
\end{align}
\end{lemma}

\begin{proof}
Define set of matrices $N_t:=\{M\in\R^{n\times n}: \norm{\C^{-1}M- I_n}_F^2\le t \}$. Define joint operator $T:=\BN\circ \F$ and Gram matrix $\Csa:=\sfrac1d T(h)T(h)^\top$. Observe that conditioned on $h$, $W' T(h)$ is equal in distribution to $W' \Csa^\half$. We can decompose the total variation based on depending on $M$ belongs to $N_t$ or not 
\begin{align}
    &\TV(h', h)\\
    &= \TV(W' \Csa^\half, W \C^\half) \\
    &\le \Prob{\Csa\notin N_t} + \sup_{\Csa\in N_t}\TV( W \Csa^\half, W \C^\half)\\
    &\le \Prob{\norm{\C^{-1}\Csa-I_n}_F^2 \ge t} +  \sup_{\Csa\in N_t}\TV(\Normal_{\Csa},\Normal_{\C})\\
    & \le n e^{- \frac{ \nfrac{t d}{2} }{\norm{\C^{-1}}\distort^2 (1+ 1/\sqrt{n}) n^2}} + \frac{3}{2}t=:\delta(t)
\end{align}
where in the last line we use the bound on total variation between $N(0,\C)$ and $N(0,M^\top M)$ from Theorem 1 of~\cite{devroye2018total} over $N_t$. Because the $\delta(t)$ is true for all $t\in[0,1]$, we differentiate to find the minimum. Define $c:=\norm{\C^{-1}}\distort^2(1+\sqrt{t/n})$. We have:
\begin{align}
    \frac{\partial\delta(t)}{\partial t} &= -n\frac{ \nfrac{d}{2} }{c n^2}e^{- \frac{ \nfrac{t d}{2} }{c n^2}} + \frac{3}{2}=0\\
   e^{- \frac{ \nfrac{t d}{2} }{c n^2}}&=\frac{3nc}{d} \\
  \hat t &=\frac{2n^2c}{d}\ln\left(\frac{d}{3nc}\right)
\end{align}
We have
\begin{align}
\delta(\hat t)=\frac{3n^2c}{d} + \frac{3n^2c}{d}\ln\left(\frac{d}{3nc}\right)\\
=\frac{3n^2c}{d}\left(1 + \ln\left(\frac{d}{3nc}\right)\right)
\end{align}
where $e$ denotes $\exp(1)$.
Putting it all together and plugging value of $c$, we have 
\begin{align}
    \TV(h',h)\le \frac{3n^2\distort^2\norm{\C^{-1}}}{d}\left(1 + \ln\left(\frac{d}{3n\distort^2\norm{\C^{-1}}}\right)\right)
\end{align}
\end{proof}

\subsection{Bounding total variation $G_\ell$ using geometric contraction}

\begin{lemma}\label{lem:tv_geometric_contraction}
Assume that Markov chain $\{h_\ell\}_{\ell\in\N}$ obeys Assumption~\ref{ass:rapid_mixing} with $\alpha>0,$ and $\C$ is the stable Gram matrix assumed to be non-degenerate. Define candidate distribution $\mu_C$ for $h_C$ by drawing its rows from $\Normal(0,\C).$
We have 
\begin{align}
    \tv{\ell}{C}\le e^{-\ell\alpha} + \eps^2
\end{align}
where $\mu_\ell$ denotes law of $h_\ell$, and with $\eps$ was defined in Lemma~\ref{lem:tv_one_step}.
\end{lemma}

\begin{proof}[Proof of Lemma~\ref{lem:tv_geometric_contraction}]
As a direct consequence of the geometric assumption, we have
\begin{align}
    \tv{\ell}{*}\le (1-\alpha)^\ell\tv{0}{*} \le e^{-\alpha},
\end{align}
where we used $1-x\le \exp(-x)$ for all $x.$ Invoking Lemma~\ref{lem:tv_one_step}, we have $\tv{C}{*}\le \eps^2$. By triangle inequality for total variation we can conclude the claim of the lemma. 

\end{proof}



%% file: sections/conclusion.tex
\section{Limitations and Future Directions}
In this paper, we presented a theoretical framework that bridges the gap between the mean-field theory of neural networks with finite and infinite widths, with a focus on batch normalization at initialization. Many questions that were out of the scope for this study, suggesting directions for new lines of inquiry.

\paragraph{Rapidly mixing assumption.}
One limitation of our work is the rapidly mixing assumption that was used to establish the concentration of our results. While our experiments validated our results based on this assumption, it would be beneficial to prove that this assumption holds for a wide range of neural networks with batch normalization.

\paragraph{Training and optimization.}
While our focus of the current work was on random neural networks, ~\citet{feng2022rank} demonstrate that the rank of input-output Jacobian of neural networks without normalization at initialization diminishes at an exponential rate with depth (Theorem 5), which implies changes in the input does not change the direction of outputs. In a remarkable observation, ~\citet{yang2019mean} show the exact opposite for BN-MLP using a mean-field analysis (Theorem 3.10): any slight changes in the input lead to considerable changes in the output. These results naturally raise the following question: Can we arrive at non-trivial results about input-output Jacobian at the infinite depth finite width regime? 

The mean-field approach is also used to analyze the training mechanism. In particular,~\citet{bach2021gradient} prove that gradient descent globally converges when optimizing single-layer neural networks in the limit of an infinite number of neurons. Although the global convergence does not hold for standard neural networks, insights from this mean-field analysis can be leveraged in understanding the training mechanism. For example, \citet{hadi22} proves the global convergence of gradient descent holds for specific neural networks with a finite width, and two dimensional inputs in a realizable setting.  

\paragraph{Exploring other normalizations.}
More research is needed for other normalization techniques, such as weight normalization~\cite{salimans2016weight} or layer normalization~\cite{ba2016layer} to understand the impact of these normalization techniques on the robustness and generalization of neural networks. Our findings highlight the power of mean-field theory for analyzing neual networks with normalization layers.

\paragraph{Extending to other architectures}
Our analyses are limited to MLPs. Extending our work to convolutional neural networks and transformers would enable us to analyze and enhance initialization for these neural networks. In particular, recent studies have shown that transformers suffer from the rank collapse issue when they grow in depth~\cite{anagnostidissignal}. A non-asymptotic mean-field theory may enable us to tackle this issue by providing a sound understanding of representation dynamics in transformers.

Overall, our results demonstrate that depth is not necessarily a curse for mean-field theory, but can even be a blessing when neural networks have batch normalization. The inductive bias provided by batch normalization controls the error propagation of mean-field approximations, enabling us to establish non-asymptotic concentration bounds for mean-field predictions. This result underlines the power of mean-field analyses in understanding the behavior of deep neural networks, thereby motivating the principle development of new initialization and optimization techniques for neural networks based on mean-field predictions. 
\section*{Acknowledgments and Disclosure of Funding}
Amir Joudaki is funded through Swiss National Science Foundation Project Grant \#200550 to Andre Kahles.
Hadi Daneshmand received funds from the Swiss National Science Foundation for this project (grant P2BSP3\_195698).  Also, we  acknowledge
support from the European Research Council (grant SEQUOIA 724063) and the French government under management of Agence Nationale de la Recherche as part of the “Investissements d’avenir” program, reference ANR-19-P3IA-0001(PRAIRIE 3IA Institute).
